%% file: main.tex
\pdfoutput=1
\relax
\documentclass[letterpaper]{article} %
\usepackage[preprint,nonatbib]{neurips_2021}  %
\usepackage{times}  %
\usepackage{helvet}  %
\usepackage{courier}  %
\usepackage[hyphens]{url}  %
\usepackage{graphicx} %
\usepackage{amsmath}
\usepackage{booktabs}
\usepackage{caption} %
\DeclareCaptionStyle{ruled}{labelfont=normalfont,labelsep=colon,strut=off} %
\usepackage{algorithm}
\usepackage{algorithmic}
\usepackage{ekzhang}
\usepackage{subfig}
\usepackage{bm}

\usepackage{newfloat}
\usepackage{listings}
\floatstyle{ruled}
\newfloat{listing}{tb}{lst}{}
\newcommand{\idim}{\textsc{dim}}
\floatname{listing}{Listing}

\title{Intrinisic Gradient Compression for Federated Learning}
\author{%
  Luke Melas-Kyriazi\thanks{Equal contribution} \\
  Department of Computer Science\\
  Oxford University\\
  \texttt{luke.melas@sjc.ox.ac.uk} \\
 \And
 Franklyn Wang$^{*}$ \\
 Harvard University \\
 Department of Mathematics\\
 Cambridge, MA 02138 \\
 \texttt{franklyn\_wang@college.harvard.edu} \\
}

\begin{document}

\maketitle

\input{abstract}

\input{intro}

\input{prelim}

\input{family}
\input{algos}

\input{exps}

\input{concl}

\clearpage

\bibliographystyle{unsrt}
\bibliography{biblio}

\clearpage

\onecolumn
\begin{center}
{\Large \textbf{Appendix}}
\end{center}
\appendix 

\input{proofs}

\input{algo3}
\input{additional_exps}

\input{gradient_reconstruction}

\end{document}

%% file: abstract.tex
\begin{abstract}
Federated learning is a rapidly-growing area of research which enables a large number of clients to jointly train a machine learning model on privately-held data. One of the largest barriers to wider adoption of federated learning is the communication cost of sending model updates from and to the clients, which is accentuated by the fact that many of these devices are bandwidth-constrained. In this paper, we aim to address this issue by optimizing networks within a subspace of their full parameter space, an idea known as \emph{intrinsic dimension} in the machine learning theory community. We use a correspondence between the notion of intrinsic dimension and gradient compressibility to derive a family of low-bandwidth optimization algorithms, which we call \emph{intrinsic gradient compression algorithms}. Specifically, we present three algorithms in this family with different levels of upload and download bandwidth for use in various federated settings, along with theoretical guarantees on their performance. Finally, in large-scale federated learning experiments with models containing up to 100M parameters, we show that our algorithms perform extremely well compared to current state-of-the-art gradient compression methods. 
\end{abstract}

%% file: intro.tex
\section{Introduction}

The key paradigm of federated learning is that data is stored locally on edge devices, while model updates (either gradients or weights) are communicated over a network and aggregated by a central server. 
This setup enables edge computing devices to jointly learn a model without data sharing, thereby retaining their data privacy. 
However, the issue of communication bandwidth often stands in the way of large-scale deployment of federated learning systems: it can be very costly to send model updates over a network, especially when communicating with mobile phones and edge devices. 

To reduce bandwidth requirements for federated learning, it is natural to compress model updates before sending them over the network. Previous works in this direction \cite{ajiheafield2017sparse,Sattler2020RobustAC,lin2018deep,DBLP:conf/icml/RothchildPUISB020} have explored compression schemes including Top-$K$ sparsification (i.e. taking the top $K$ weights with the largest magnitude) and gradient sketching. 

At the same time, in the machine learning theory community, researchers have been working to understand what at first seems like an entirely different question: why do hugely overparametrized models generalize so well? One promising approach to answering this question has utilized the concept of \emph{intrinsic dimension}, defined for a given optimization problem as the smallest dimension $d$ for which we can solve the problem when the weights are restricted to a a $d$-dimensional manifold. To be precise, it is the smallest $d$ for which an optimization problem \begin{equation}\label{eq:form} \min_{\theta \in \mc{M}_d} \ell(\theta) \end{equation} has a satisfactory solution, where $\mc{M}_d$ is a $d$-dimensional manifold. If the intrinsic dimension of an optimization problem is low, then even if a model is vastly overparameterized, only a small number of parameters need to be tuned in order to obtain a good solution, which is often enough to imply certain generalization guarantees. 

We begin this paper by observing that the two problems above are naturally related. If one can find a solution to the problem by only tuning $d$ parameters, as in \Cref{eq:form}, then a corresponding low bandwidth algorithm can be found by simply running gradient descent on $\mc{M}_d$. This occurs because gradients on $\mc{M}_d$ are $d$-dimensional, and hence require less bandwidth to communicate.

However, for very small $d$ (as is desired), it is often insufficient to simply optimize a $d$-sized subset of a model's parameters, especially if this subset must be chosen manually for each neural network architecture. Thus, we are inspired to seek a more general family of these types of low-bandwidth algorithms. 

We rewrite the optimization problem in \Cref{eq:form} in the original parameter space as \[ \min_{\theta' \in \R^d} \ell(f_{A\theta'}) \] 
so then stochastic gradient descent in the original space can be written as 
\begin{equation}\label{eq:standard_vanilla}
\theta_{t+1} = \theta_t - \eta AA^{\top} \nabla_{\theta} \ell(f_{\theta})|_{\theta = \theta_t}. 
\end{equation}

We call this method \emph{static intrinsic gradient compression}, because our gradients are projected into a static (``intrinsic'') subspace. Now, \Cref{eq:standard_vanilla} admits a natural generalization, which allows us to explore more of the parameter space while still preserving a low level of upload bandwidth usage: 
\begin{equation}\label{eq:standard_tv} \theta_{t+1} = \theta_t - \eta A_tA_t^{\top} \nabla_{\theta} \ell(f_{\theta})|_{\theta = \theta_t} \end{equation}
where $A_t$ may vary with time. We call the set of all such algorithms \emph{intrinsic gradient compression algorithms}, and consider three particular instantiations for federated learning: static, $K$-subspace, and time-varying intrinsic gradient compression. 

The static algorithm is an extremely simple baseline; it simply projects the local model update to a lower-dimensional space before sending it to the server to be aggregated. Nonetheless, we find that it performs remarkably well in practice compared to recent gradient compression schemes. The $K$-subspace and time-varying algorithms are designed specifically for federated learning: the $K$-subspace method reduces the upload bandwidth requirements of the static algorithm, while the time-varying method improves performance across multiple of distributed training.

Our approach is model-agnostic and highly scalable. In experiments across multiple federated learning benchmarks (language modeling, text classification, and image classification), we vastly outperform prior gradient compression methods, and show strong performance even at very high compression rates (e.g. up to $1000\times$). 

Our contributions are as follows. 

\begin{itemize}
    \item We find a general class of optimization algorithms based on the notion of intrinsic dimension that use low amounts of upload bandwidth, which we denote \emph{intrinsic gradient compression algorithms}.
    \item We specify three such algorithms: static compression, time-varying compression and $K$-subspace compression, with different levels of upload and download bandwidth for use in various federated settings.
    \item We provide theoretical guarantees on the performance of our algorithms.
    \item Through extensive experiments, we show that these methods outperform prior gradient compression methods for federated learning, obtaining large reductions in bandwidth at the same level of performance. 
\end{itemize}

%% file: prelim.tex
\section{Preliminaries}\label{sec:prelim}

\subsection{Intrinsic Dimension}

The concept of intrinsic dimension was introduced in the work of \cite{li2018measuring}, as a way of evaluating the true difficulty of an optimization problem. While this can usually be done by counting the number of parameters, some optimization problems are easier than others in that solutions may be far more plentiful. To illustrate this concept, we will take an optimization problem over a large space $\Theta^{1}$ and a small space $\Theta^{2}$ so that for any $\theta \in \Theta^{2}$, for the function $f$ we have $f(\theta
') \in \Theta_1$. If $\theta$ is in the image of $f$ on $\Theta^2$, one can write 
\begin{equation}\label{eq:subspace}
\ell(f_{\theta}) = \ell(f_{g(\theta')}) 
\end{equation} 
where $g: \Theta^2 \rightarrow \Theta^{1}$ and thus transform the original problem over $\Theta^{1}$ into an optimization problem over $\Theta^{2}$. If we can still find good solutions to the original problem where $\theta' \in \Theta^{2}$, then the problem may be easier than originally expected. Intuitively, even though the ``true" dimension of the optimization problem is $D$, the fact that good solutions can be found while searching over a manifold of dimension $d$ suggests that the problem is easier than a typical dimension $D$ optimization problem.

With this, we can now define the notion of intrinsic dimension. The intrinisic dimension $\idim(\ell, L)$ with respect to a task $\ell$ and performance threshold $L$ is equal to the smallest integer $d$ so that optimizing \Cref{eq:subspace} on task $\ell$ could lead to a solution of performance at least equal to $L$. The intrinsic dimension is not completely knowable, because we cannot find the ``best performing model'' exactly. However, if say, training with some optimization algorithm gives us a solution to \Cref{eq:subspace} with loss $\le L$ and with $d$ dimensions, we can say with certainty that $\idim(\ell, L) \le d$.

Throughout this paper we will always take $g(\theta') = A\theta' + \theta_0$ for a $D \times d$ matrix $A$, and take $\Theta^{2} = \R^{d}$, and $\Theta^{1} = \R^{D}$, where $D > d$, where $\theta_0$ is the original value of the expression. Consequently, the image of $f$ on $\Theta^2$ (and thus the dimension over which we optimize) is an affine $d$-dimensional subspace of $\R^{D}$. The affine nature is crucial -- it allows us to do a full fine-tune starting from a pretrained checkpoint, which is not possible if we just use a standard subspace.

\subsection{Related Work}

Below, we describe how our contribution relates to relevant prior work. Due to space constraints, we describe additional related work in \Cref{app:additional_related_work}. 

\paragraph{Intrinsic Dimension}

As discussed in the previous section, \cite{li2018measuring} introduced the concept of intrinsic dimensionality to gain insight into the difficulty of optimization problems.\footnote{The concept of intrinsic dimension has also been used to describe the dimensionality of datasets; these works are not directly related to ours, but we provide an overview of them in \Cref{app:additional_related_work}.} \cite{aghajanyan2020intrinsic} followed up on this work by considering the setting of finetuning models in natural language processing. They show that the intrinsic dimension of some of these tasks is surprisingly low, and claim that this result explains the widespread success of the language model finetuning. 

These works form the basis of our static intrinsic gradient compression algorithm. Whereas these works use the concept of intrinsic dimension as a mechanism for understanding optimization landscapes, we use it as a tool for gradient compression. We then extend these works by introducing two new algorithms designed for the federated setting: $K$-subspace and time-varying intrinsic dimension. Our algorithms were not explored by previous works because they are uniquely interesting from the perspective of federated learning: they are designed to reduce communication bandwidth rather than to shed insight into objective landscapes. 

\paragraph{Gradient Compression}
With the proliferation of large-scale machine learning models over the past decade, the topic of distributed model training has gained widespread attention. Federated learning combines the challenges of distributed training and limited network bandwidth, motivating the use of gradient compression. For example, a single gradient update for a 100 million parameter model takes approximately 0.4 gigabytes of bandwidth (uncompressed). 

Gradient compression methods may be divided into two groups: biased and unbiased methods. Unbiased gradient compression estimators tend to be more straightforward to analyze, and are generally better understood for stochastic gradient descent. As long as their variance is bounded, it is usually possible to obtain reasonable bounds on their performance. Biased gradient compression estimators are typically much more challenging to analyze, although they often deliver good empirical performance. 
For example, top-$K$ compression is a popular (biased) method which takes the $k$ elements of the gradient with largest magnitudes. Numerous papers are dedicated to the topic of debiasing such methods to make them more amenable to theoretical analysis. In particular, many of these use the idea of error feedback \cite{stich2020error, ef21} to obtain theoretical guarantees on otherwise biased algorithms, like Top-K \cite{lin2018deep} and FetchSGD \cite{DBLP:conf/icml/RothchildPUISB020}. Other more exotic alternative ideas also exist, like \cite{albasyoni2020optimal}, which finds an optimal gradient compression algorithm, albeit one which is computationally infeasible.

\paragraph{Federated and Distributed Learning} 
From the introduction of federated learning \cite{mcmahan2017communication}, it was clear that communication costs represented a significant challenge to its widespread adoption. \cite{mcmahan2017communication} introduced the FedAvg algorithm, which aims to reduce communication costs by performing multiple local updates before communicating model updates. However, even with local update methods such as FedAvg, communicating model updates often remains too costly.\footnote{Additionally, the benefits of these methods are vastly diminished when clients have a small amount of local data, as many rounds of communication are necessary.} As a result, the area of gradient compression has attracted recent attention within the federated learning community. 

Top-$K$ compression is among the simplest and most intuitive compression schemes. \cite{ajiheafield2017sparse} showed that top-$K$ compression with $K = 1\%$ produced good results on neural machine translation and MNIST image classification tasks. \cite{shi2019understanding} provided a theoretical analysis and an approximate top-$K$ selection algorithm to improve sampling efficiency. \cite{Sattler2020RobustAC} combined top-$K$ compression with ternary quantization and a Golomb encoding of the weight updates. \cite{konecny2018federated} study multiple strategies for improving communication efficiency, including low-rank updates, randomly masked updates, and sketched updates. Their low-rank update strategy is related to our method, but we differ from them in that we compute our low-dimensional updates differently, perform large-scale experiments, give theoretical analysis, and consider the trade-off between download and upload bandwidth (only upload bandwidth). Also related, \cite{vkj2019powerSGD} proposed a low-rank version of SGD based on power iteration for data-parallel distributed optimization. Most recently, FetchSGD~\cite{DBLP:conf/icml/RothchildPUISB020} used sketching to reduce the size of gradients before sending them over the network. FetchSGD is the current state-of-the-art in gradient compression.

Finally, it is important to note that local update methods (e.g. FedAvg) and gradient compression methods may be combined. In particular, one can simply perform multiple training steps before compressing resulting the model update ($\theta^{\text{final}}_{\text{local}} - \theta^{\text{initial}}$). For fair comparison to FetchSGD, in our experiments, we only perform one local step per update.

%% file: algos.tex
\section{Methods}\label{sec:fedgradient}

\subsection{Intrinsic Gradient Compression}

In this subsection, we characterize a family of low-bandwidth optimization algorithms based on the notion of intrinsic dimension. In the following subsection, we will describe three algorithms from this family in detail, which we implemented

We start from the optimization problem induced by intrinsic dimension (\Cref{eq:subspace}). If we directly run gradient descent on \Cref{eq:subspace} with respect to the intrinsic weights $\theta'$, we obtain an equation of the following form: 
\begin{align*} 
\theta_{t+1}' 
&= \theta_{t}' - \eta \nabla_{\theta'} \left( \ell (f_{g(\theta')}) \right) = \theta_{t}' - \eta \nabla_{\theta'} \left( \ell (f_{A \theta'}) \right) \\
&= \theta_{t}' - \eta A^{\top}\nabla_{\theta}(\ell (f_{\theta}))^{\top}|_{\theta=A\theta'_t+\theta_0}
\end{align*}

Then, left-multiplying both sides by $A$ we obtain 
\begin{equation}\label{eq:gradcompress}
    \theta_{t+1} = \theta_t - \eta \underbrace{A \underbrace{A^{\top} \nabla_{\theta}(\ell(f_{\theta}))|_{\theta = \theta_t}}_{\text{compressed gradient}}}_{\text{approximate gradient}}
\end{equation} 

Note that here, we can interpret $A^{\top} \nabla_{\theta} (\ell(f(\theta)))|_{\theta = \theta_t}$ as a compressed gradient with dimension $d$, and $AA^{\top}\nabla_{\theta} (\ell(f(\theta)))|_{\theta = \theta_t}$ as the approximate gradient. This inspires us to consider the more general family of optimization algorithms given by
\begin{equation}\label{eq:general}\theta_{t+1} = \theta_t - \eta A_t A_t^{\top} (\bm{v}_t), 
\end{equation}
where $\bm{v}_t$ is a $D$ dimensional vector computed from data available at timestep $t$ that plays a similar role to a gradient, but may not be an exact gradient, and the $A_t$ are all $D \times d$ matrices known ahead of time (say, generated with random seeds). One intuitive way of interpreting this algorithm is that $\theta_{t+1} - \theta_t$ is constrained to lie in a low-dimensional subspace, namely that given by the span of $A_t$. This family of algorithms can be made to use only $d$ upload bandwidth, as only the vector $A_t^{\top}(\bm{v}_t)$ must be uploaded. Furthermore, note that \Cref{eq:general} has no references to the intrinsic weights $\theta'$, meaning that it represents a general optimization algorithm in the original space. Formally,
\begin{proposition}\label{thm:lowupload}
All optimization algorithms of the form \[  \theta_{t+1} = \theta_t - \eta A_t A_t^{\top} (\bm{v}_t) \] can be simulated with $d$ upload bandwidth in a standard federated learning setting, where $\bm{v}_t$ is a function that can be calculated by the client at time $t$ combined with all data from the server, and $A_t$ is a $D \times d$ matrix known to both the client and the server. 
\end{proposition}

We call all algorithms of the form above \emph{intrinsic gradient compression algorithms}. 

\begin{table*}
\renewcommand{\arraystretch}{1.2}
\centering
\begin{tabular}{l | c | c | c }
Intrinsic Gradient Compression Method & Upload & Download & Dimensions Explored  \\
\hline \hline
No Compression                  & $DE$      & $DE$          & $D$   \\
\hline
Static                          & $dE$      & $dE$          & $d$   \\ 
Time-Varying                    & $dE$      & $2dE$         & $dE$  \\ 
$K$-Subspace                    & $dE$      & $dEK$         & $dK$  \\
$K$-Subspace + Time-Varying     & $dE$      & $2dEK$        & $dEK$ \\
\end{tabular}
\vspace{-2mm}
\caption{Bandwidth and Performance Comparisons. The bandwidth refers to that of that used for each client. Note that we break upload and download bandwidth into separate columns, because download speeds can often be considerably faster than upload speeds and we may thus be willing to tolerate higher values of download bandwidth. A realistic example of the values of the variables above is e.g. $d = 10^{3}, D = 10^{8}, E = 20, K = 8$.}
\vspace{-4mm}
\label{tbl:tradeoffs}
\end{table*}

\subsection{Algorithms}

While \Cref{thm:lowupload} shows that any algorithm of the form \Cref{eq:general} can be implemented with low levels of upload bandwidth, not every algorithm of the form \Cref{eq:general} can be implemented with low levels of download bandwidth as well. In this section, we describe three particular intrinsic gradient compression algorithms which use low amounts of both upload and download bandwidth. We show the theoretical tradeoffs between each of these algorithms in \Cref{tbl:tradeoffs}. 

These federated learning algorithms can be decomposed into three main phases.
\begin{itemize}
    \item \textbf{Reconciliation:} The client reconciles its model with the server's copy of the model.
    \item \textbf{Compression:} The local model calculates, compresses, and sends its local gradient to the server.
    \item \textbf{Decompression:} The server updates its own copy of the model using the estimated gradients it has received. 
\end{itemize}

Compression and decompression are shared between all algorithms, while each algorithm has a distinct reconciliation phase. 

\paragraph{Static Intrinsic Gradient Compression}

The static intrinsic gradient compression simply involves projecting gradients into a fixed (``static'') low-dimensional space and reconstructing them on the server: 
\[  \theta_{t} = \theta_{t-1} - \eta AA^{\top} \nabla_{\theta} \mc{L}(\theta_{t-1}) \]
Nonetheless, it performs remarkably well in practice (see \Cref{sec:exps}). The full algorithm is given in Algorithm~\ref{alg:FedSSC}.

Note that in the reconciliation phase, the parameters $\theta^{c}$
(which are on the server)
will always be equal to $\theta_0 + A\Sigma$ for some $\Sigma \in \R^{d}$. Thus, the server can just send $\Sigma$ to the client, using $d$ download bandwidth. 
In the compression phase, the client compresses the gradient by multiplying by $A^{\top}$, and for decompression the server multiplies this by $A$. 
The client then compresses the gradient by multiplying by $A^{\top}$, and the server decompresses it by multiplying it by $A$. 

\begin{algorithm}[t]
\small
\caption{Static Intrinsic Gradient Compression}
\begin{algorithmic}
\STATE \textbf{input:} learning rate $\eta$, timesteps $T$, local batch size $\ell$, clients per round $W$
\STATE Create matrix $A \in \R^{D \times d}$ with $\BE[AA^{\top}] = I_D$. Spawn $A$ on all nodes using a suitable random number generator. 
\STATE Current Vector: $\Sigma_{0} = 0$
\FOR{$t = 1, 2 \cdots T$}
\STATE Randomly select $W$ clients $c_1, \ldots c_W$.
\LOOP\STATE{\{In parallel on clients $\{c_i\}_{i=1}^{W}$\}} 
\STATE Download $\Sigma_{t - 1}$, calculate current $\theta_{t-1} = \theta_0 + A(\Sigma_{t - 1}) $.
\STATE Compute stochastic gradient $g_{i}^{t}$ on batch $B_i$ of size $\ell$: $g_{i}^{t} = \frac{1}{\ell} \sum_{j=1}^{\ell} \nabla_{\theta} \mathcal{L}(\theta_{t-1}, z_j)$ where $B_i = \{z_j\}_{j=1}^{\ell}$. 
\STATE Sketch $g_{i}^{t}$ to $S_i^{t} = A^{\top}g_{i}^{t}$ and upload it to the aggregator.
\ENDLOOP
\STATE Aggregate sketches $S^{t} = \frac{1}{W} \sum_{i=1}^{W} S_i^{t}$
\STATE Unsketch: $\Delta_{t} = AS^{t}$
\STATE Update: $\theta_{t} = \theta_{t - 1} - \eta\Delta_{t}$, $\Sigma_{t} = \Sigma_{t - 1} - \eta S^{t}$.
\ENDFOR
\end{algorithmic}
\label{alg:FedSSC}
\end{algorithm}

\paragraph{$K$-Subspace Static Intrinsic Gradient Compression}

The $K$-subspace algorithm is motivated by the fact that in some cases, upload bandwidth is more heavily constrained than download bandwidth. Rather than using a single compression matrix $A$, we use a set of $K$ different compression matrices $\{A^{(i)}\}_{i=1}^{K}$, each corresponding to a different subspace. At each iteration, each client is randomly assigned one of these $K$ matrices. Each client then explores a subspace of dimension $d$ and uploads a vector of size $d$ to the server. Finally, the server aggregates these local updates into a global update of size $dK$, which is downloaded by each client. In this way, it is possible to explore a subspace of size $dK$ using only $d$ upload bandwidth. With $K=1$, this algorithm is equivalent to static gradient compression. The full algorithm is given in Algorithm~\ref{alg:FedkTVSC}.
\begin{algorithm}[t]
\footnotesize
\vspace{1mm}\vspace{1mm}
\caption{$K$-Subspace Intrinsic Gradient Compression}
\begin{algorithmic}
\STATE \textbf{input:} distinct subspaces $K$, learning rate $\eta$, timesteps $T$, local batch size $\ell$, clients per round $W$
\STATE Create matrices $A^{(1)}, A^{(2)}, \ldots A^{(K)} \stackrel{\text{i.i.d.}}{\sim} A$ where $A \in \R^{D \times d}$ with $\BE[AA^{\top}] = I_D$. Spawn across all nodes using a random seed $s_t$ which is distinct but generates one of $A^{(1)}, A^{(2)}, \ldots A^{(K)}$.
\STATE Current Vector: $\Sigma^{\mathrm{current}(k)} = 0$ for $k = 1, 2, \ldots K$.
\FOR{$e = 1, 2, \ldots E$}

\FOR{$t = 1, 2 \cdots T$}
\STATE Randomly select $W$ clients $c_1, \ldots c_W$.
\LOOP\STATE{\{In parallel on clients $\{c_i\}_{i=1}^{W}$\}} 
\STATE Download $\Sigma^{\mathrm{current}(k)}$ for $k = 1, \ldots K$, calculate current 
\STATE \[ \theta^{c_i}_e = \theta_0 + \sum_{k=1}^{K} A^{(k)} \Sigma^{\text{current}(k)} \]
\STATE Choose a random $k_1 \sim \text{DUnif}(\{1, 2, \ldots K\})$ 
\STATE Compute stochastic gradient $g_{i}^{t}$ on batch $B_i$ of size $\ell$: $g_{i}^{t} = \frac{1}{\ell} \sum_{j=1}^{\ell} \nabla_{\theta} \mathcal{L}(\theta_{e}^{c_i}, z_j)$ where $B_i = \{z_j\}_{j=1}^{\ell}$. 
\STATE Sketch $g_{i}^{t}: S_i^{(e)t} = (k_1, A^{(k_1)\top}g_{i}^{t})$ and upload it to the aggregator.
\ENDLOOP
\STATE Write sketches received as $\{S^{(e)t}_w\}_{w=1}^{W} = \{(j_w, C_w^{(e)t})\}_{w=1}^{W}$.
\STATE Unsketch $S^{(e)t}$ to get $\Delta^{(e)t} = \frac{1}{W}\sum_{w=1}^{W} A^{(j_w)} C^{(e)t}_w $
\STATE Update: $\theta^{\mathrm{current}} = \theta^{\mathrm{current}} - \eta\Delta^{(e)t}$, 
\FOR{$k = 1, 2 \ldots K$}
\STATE Update: $\Sigma^{\mathrm{current}(k)} = \Sigma^{\mathrm{current}(k)} - \frac{\eta}{W} \sum_{j_w = k} C_w^{(e)t} $.
\ENDFOR
\ENDFOR
\ENDFOR
\end{algorithmic}
\vspace{1mm}\vspace{1mm}
\label{alg:FedkTVSC}
\end{algorithm}

\paragraph{Time-Varying Intrinsic Gradient Compression}

Finally, the time-varying algorithm utilizes the fact that changing the subspace in which we are optimizing is nearly costless: it simply involves sending the random seed $s_i$ from which the (pseudo-)random matrix $A_i$ may be generated. Rather than using one (or a set of) static compression matrices for all epochs (i.e. one round of training over all clients), we generate a new matrix $A_i$ at each epoch $i$. Formally, we have:
\[ \theta_t = \theta_{t-1} - \eta A_{e}A_{e}^{\top} \nabla_{\theta} \mc{L}(\theta_{t-1}) \]
In this case, our algorithm can be implemented with at most $2d$ bandwidth used per client per timestep, so over $E$ epochs there is $2dE$ bandwidth used total on downloading. Since this bandwidth is twice that of static subspace compression, but we search $E$ times more directions in the space, this algorithm is particularly useful when we have many epochs.

Letting $\theta_{e}^{c}$ be the client parameters at epoch $e$, 
note that we have the value of $\theta_{e-1}^{c}$ when performing reconciliation. Now we can write 
\[ \theta_{e}^{c} - \theta_{e-1}^{c} = (\theta_{e}^{c} - \theta_{e-1}^{\text{final}}) + (\theta_{e-1}^{\mathrm{final}} - \theta_{e-1}^{c}) \]
We can see that $(\theta_{e}^{c} - \theta_{e-1}^{\text{final}})$ lies in the span of $A_e$ and $(\theta_{e-1}^{\text{final}} - \theta_{e-1}^{c})$ lies in the span of $A_{e-1}$, showing the validity of the algorithm, which is given in full in Algorithm~\ref{alg:FedTVSC}.

Finally, we note that it is possible to use both $K$-subspace and time-varying compression together. In this case, a new batch of $\{A_e^{(i)}\}_{i=1}^{K}$ of $K$ compression matrices is generated at each epoch $e$. We do not experiment with this setup, but it is likely to show further improvements over using each of these methods alone. 

\begin{algorithm}[t]
\footnotesize
\caption{Time-Varying Intrinsic Gradient Compression}
\begin{algorithmic}
\STATE \textbf{input:}  learning rate $\eta$, timesteps $T$, local batch size $\ell$, clients per round $W$
\FOR{$e = 1, 2, \ldots , E$}
\STATE Create matrix $A_e \stackrel{\text{i.i.d.}}{\sim} A$ where $A \in \R^{D \times d}$ with $\BE[AA^{\top}] = I_D$, and spawn it on all nodes. 
\STATE Current, Final Vector: $\Sigma^{\mathrm{current}}_{e} = 0$, $\Sigma^{\mathrm{final}}_{e} = 0$
\FOR{$t = 1, 2 \ldots ,T$}
\STATE Randomly select $W$ clients $c_1, \ldots c_W$.
\LOOP\STATE{\{In parallel on clients $\{c_i\}_{i=1}^{W}$\}} 
\STATE Download $\Sigma^{\mathrm{current}}_e, \Sigma^{\mathrm{final}}_{e-1}$, calculate current $\theta^{c_i}_e = \theta^{c_i}_{e-1} + A_{e-1}(\Sigma_{e - 1}^{\mathrm{final}} - \Sigma^{\mathrm{last}}) + A_e(\Sigma^{\mathrm{current}}_e)$.
\STATE Update $\Sigma^{\mathrm{last}} = \Sigma^{\mathrm{current}}_e$.
\STATE Compute stochastic gradient $g_{i}^{t}$ on batch $B_i$ of size $\ell$: $g_{i}^{t} = \frac{1}{\ell} \sum_{j=1}^{\ell} \nabla_{\theta} \mathcal{L}(\theta_{e}^{c_i}, z_j)$ where $B_i = \{z_j\}_{j=1}^{\ell}$.
\STATE Sketch $g_{i}^{t}: S_i^{(e)t} = A_e^{\top}g_{i}^{t}$ and upload it to the aggregator.
\ENDLOOP
\STATE Aggregate sketches $S^{(e)t} = \frac{1}{W} \sum_{i=1}^{W} S_i^{(e)t}$
\STATE Unsketch: $\Delta^{(e)t} = A_e S^{(e)t}$
\STATE Update: $\theta^{\mathrm{current}} = \theta^{\mathrm{current}} - \eta\Delta^{(e)t}$, $\Sigma_e^{\mathrm{current}} = \Sigma_{e}^{\mathrm{current}} - \eta S^{(e)t}$.
\ENDFOR
\STATE Let $\Sigma_{e}^{\mathrm{final}} = \Sigma_{e}^{\mathrm{current}}$.
\ENDFOR
\end{algorithmic}
\label{alg:FedTVSC}
\end{algorithm}

\paragraph{Choice of Compression Matrix}\label{sec:fedgradient_choice}

Here, we discuss how to choose $A$. Our methods are theoretically agnostic to the choice of $A$, and depend only on the existence of efficient subroutines for calculating the matrix-vector products $Ax$ and $A^{\top}y$. Nonetheless, the choice of $A$ has significant practical considerations, which we discuss here.

The naive choice is to let $A$ be a $D \times d$ random dense matrix, but such a choice is impossible due to memory constraints. For example, if we aim to train even a small version of BERT (100M parameters) with an intrinsic dimension of $1000$, we would need to store a matrix with $10^{11}$ entries. 

Our approach, also taken by \cite{aghajanyan2020intrinsic, li2018measuring}, utilizes the \textit{Fastfood transform} \cite{DBLP:conf/icml/LeSS13}. This transform expresses the $D \times d$ matrix $A_i$ as $ A_i = \text{Unpad}_DB_iH\Pi_i G_iH\text{Pad}_{2^{\ell}}$ where $2^{\ell}$ is the smallest power of two larger than $D$, $H$ is a standard Hadamard matrix, $B_i$ is a random diagonal matrix with independent Rademacher entries (random signs), $\Pi$ is a random permutation matrix, $G$ is a random diagonal matrix with independent standard normal entries, $\text{Pad}_{2^{\ell}}$ to be a linear operator which simply pads a $d$-dimensional vector $v$ with zeroes until it has size $2^{\ell}$, and $\text{Unpad}_{D}$ is a linear operator which takes the first $D$ elements from a $2^{\ell}$-dimensional vector. Since we can quickly compute a matrix-vector product by $H$ with a fast Walsh-Hadamard transform, we can perform a matrix multiplication by $A_iA_i^{\top}$ in $O(\ell2^{\ell}) = O(D\log D)$ time and $O(D)$ space. 

Finally, to ensure that we do not need to communicate the matrices $A_i$, we generate each matrix pseudorandomly from a random seed $s_i$. Thus, the matrices $A_i$ do \textit{not} need to be transferred over the network.

\subsection{Theoretical Guarantees}

In this section, we provide guarantees on static, time-varying, and $K$-subspace intrinsic gradient compression. We focus on convex functions, which are the most amenable to analysis. First, we contend that it is not interesting to prove guarantees of the form 
``time-varying intrinsic gradient compression works well for \emph{all convex functions}''.
This is because the hypotheses are too weak to produce meaningful results, even if one assumes that one has access to oracle convex optimization routines which return the minimizer (rather than just an approximate optimizer).   %

Two representative works, similar to ours, which consider a setup where we have access to an oracle which finds minimizers of convex functions are \cite{stich2013optimization} and \cite{ssobound}. \cite{stich2013optimization} considers an optimization algorithm which searches over random $1$-dimensional subspaces, showing that theoretically, searching $1$ random direction $n$ times performs about as well as searching $n$ directions once, offering no bandwidth benefit in our context. \cite{ssobound} shows a similar result without requiring random subspaces. Thus, showing interesting guarantees for arbitrary convex functions is likely quite challenging.

Rather, in the flavor of intrinsic dimension, we assume that 
our convex optimization problems are ``easier" than standard problems, in that 
searching few directions is likely to yield good solutions. 
In this case, we show that time-varying intrinsic dimension works even better than static compression.
Intuitively, this is because each random subspace sampled in the time-varying algorithm contains a point which allows us to meaningfully reduce our loss. As a consequence, when we consider many subspaces sequentially, we can reduce our loss exponentially. 

Thus, we state our hypotheses via a formalized definition of intrinsic dimension.
\begin{definition}
A convex function $g: \mathbb{R}^{D} \rightarrow \mathbb{R}$ has \textit{intrinsic dimension} $(\delta, d, \rho)$ if for all $\theta_0$ we have \[ \mathbb{P}\pa{\min_{e \in \mc{H}} g(\theta_0 + e) - g^{\star} \le \rho(g(\theta_0) - g^{\star})} \ge 1 - \delta   \] 
where $\mc{H}$ is a uniformly chosen $d$-dimensional subspace over the Grassmanian, and $g^{\star}$ is 
the minima of the function $g$. 
\end{definition}

The result on static compression now follows directly. We merely need to account for the fact that we are using an approximate optimization algorithm and not an oracle optimization algorithm. However, since a convex problem on a subspace is convex, this follows directly from well-known guarantees on gradient descent.

In what follows, we assume that from each step we have access to $\bm{g}_t$, an unbiased estimate of the true gradient of $g$ at time $t$, given the current $\theta$ we have -- such a $\bm{g}_t$ naturally emerges from our methods, where the randomness comes from the data points in the batch. In all cases, we assume that $A$ is an orthonormal basis of a random subspace sampled according to the Grassmanian. All proofs are given in \Cref{appa:proofs}. 

\begin{theorem}\label{thm:static}
For the static compression algorithm, if the function $g$ has intrinsic dimension $(\delta, d, \rho)$, we have \[ \mathbb{P}\pa{g(\hat{\theta}) - g^{\star} \le \rho(g(\theta_0) - g^{\star}) + \epsilon} \ge 1 - \delta \] if we take $\tilde{O}(\sigma^2 / \epsilon^2)$ total steps where $\hat{\theta}$ is obtained by running the static compression algorithm, and $\sigma^2 = \mathrm{Var}(A^{\top} \bm{g}_t)$. 
\end{theorem}

For $K$-subspace compression, we do not obtain stronger theoretical guarantees than static, but we include the result for completeness. Note that they use the same amount of upload bandwidth total, because $K$-varying saves a factor of $K$ on upload. We also need a further assumption on the ratio of the variance to the squared mean: if it is too small, the extra variance induced by the $K$-varying method causes the performance drop to be substantial. 

\begin{theorem}\label{thm:kvary}
For the $K$-subspace algorithm, if the function $g$ has intrinsic dimension $(\delta, d, \rho)$ with probability $1 - \delta$, we have \[ \mathbb{P}\pa{g(\hat{\theta}) - g^{\star} \le \rho(g(\theta_0) - g^{\star}) + \epsilon} \ge 1 - \delta \] if we take $\tilde{O}(K(1 + 1 / C)\sigma^2 / \epsilon^2)$ steps, where $\sigma^2 = \mathrm{Var}(A^{\top}\bm{g}_t)$, assuming that $\frac{\mathrm{Var}(A^{\top}\bm{g}_t)}{ \norm{\mathbb{E}[(A^{\top}\bm{g}_t)]}^2} \ge C$ for all values of $\theta$ for some $C > 0$ and $A$ is defined as $\begin{bmatrix} A^1 & A^2 & \ldots & A^k \end{bmatrix}$. 
\end{theorem}

Finally, we prove a better guarantee for time-varying compression, taking advantage of effectively exponential decaying loss from repeatedly applying \Cref{thm:static}.

\begin{theorem}\label{thm:timevary}
For the time-varying algorithm, if the function $g$ has intrinsic dimension $(\delta, d, \rho)$ over $E$ epochs, \[ \mathbb{P}\pa{ g(\hat{\theta}) - g^{\star} \le \rho^{E}(g(\theta_0) - g^{\star}) + \frac{\epsilon\sqrt{E}}{1 - \rho}} \ge (1 - \delta)^{E} \] after taking $\tilde{O}(\sigma^2 / \epsilon^2)$ steps, where 
$\sigma^2 = \max(\mathrm{Var}[A_1\bm{g}_t], \ldots ,\mathrm{Var}[A_E\bm{g}_t])$
\end{theorem}

%% file: exps.tex
\begin{figure}[t!]%
    \centering
    \subfloat[\centering Accuracy on CIFAR-10 across compression rates. ]{{\includegraphics[width=0.42\textwidth]{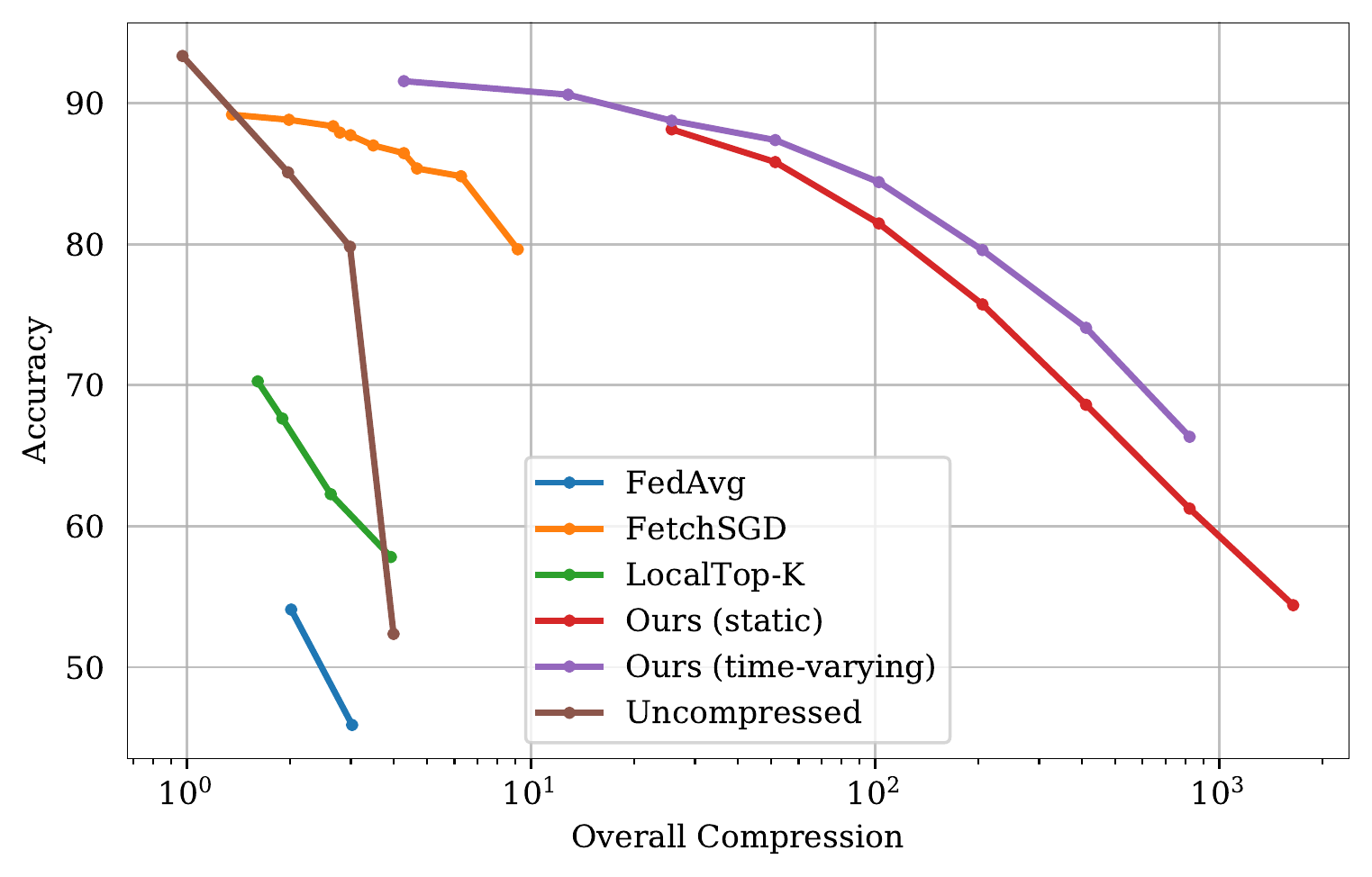}}}%
    \qquad
    \subfloat[\centering Training curves on CIFAR-10 of static and time varying compression for the intrinsic dimension $d=2000$. \vspace{-2mm} ]{{\includegraphics[width=0.42\textwidth]{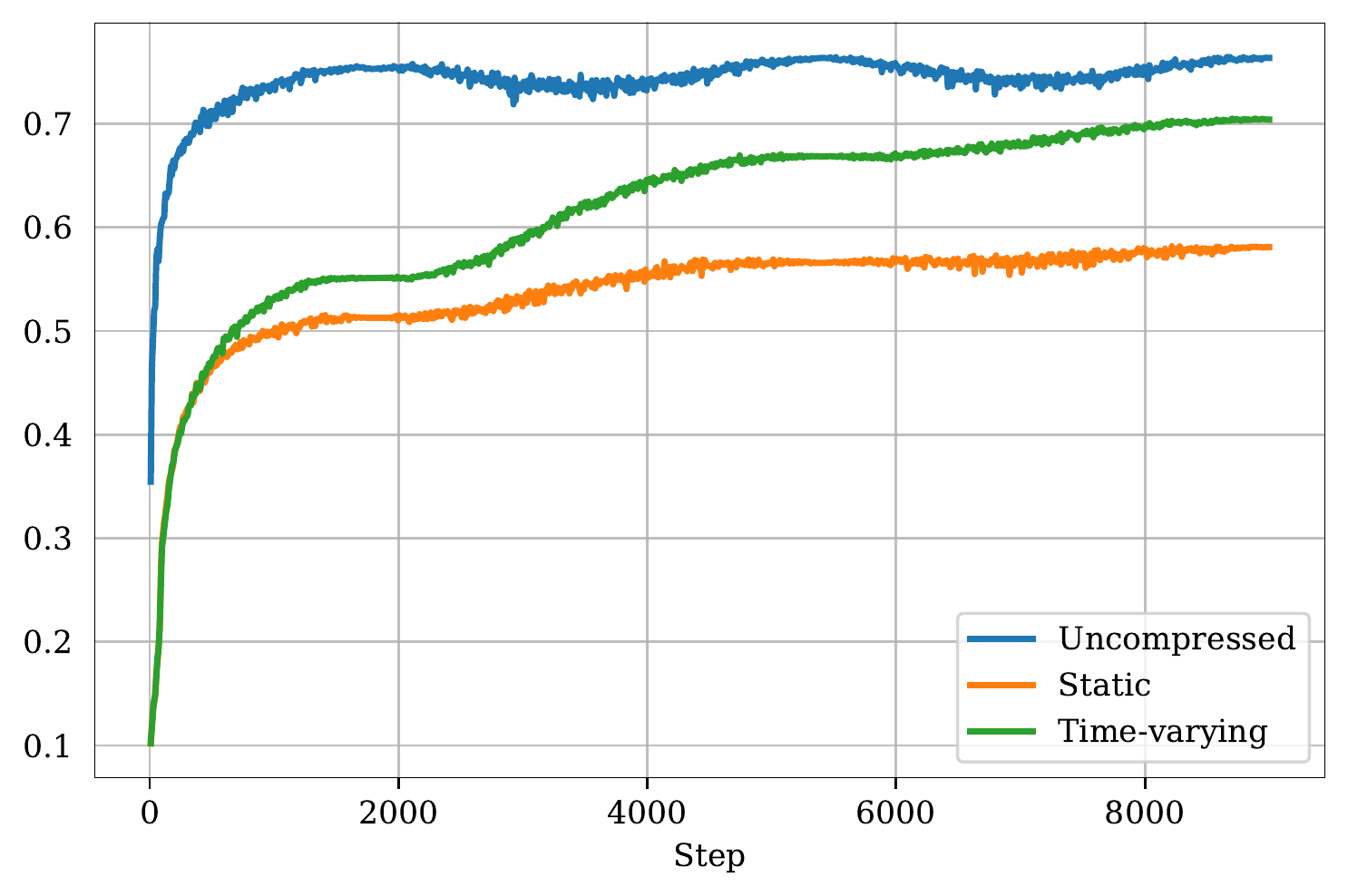}}%
}%
    \caption{Results on computer vision benchmarks. Both static and time-varying intrinsic gradient dimension significantly outperform prior work, with time-varying intrinsic compression performing best. On the right, we see that time-varying and static compression perform similarly at the beginning of training, but time-varying outperforms static with equal space when the compression is higher. For the FedAvg and uncompressed methods with compression rates above 1, compression was performed by training for fewer epochs.}
    \label{fig:cvfig}
    \vspace{-6mm}
\end{figure}

\begin{figure}[h]%
    \centering
    \subfloat[\centering Perplexity on PersonaChat  ]{{\includegraphics[width=0.4\textwidth]{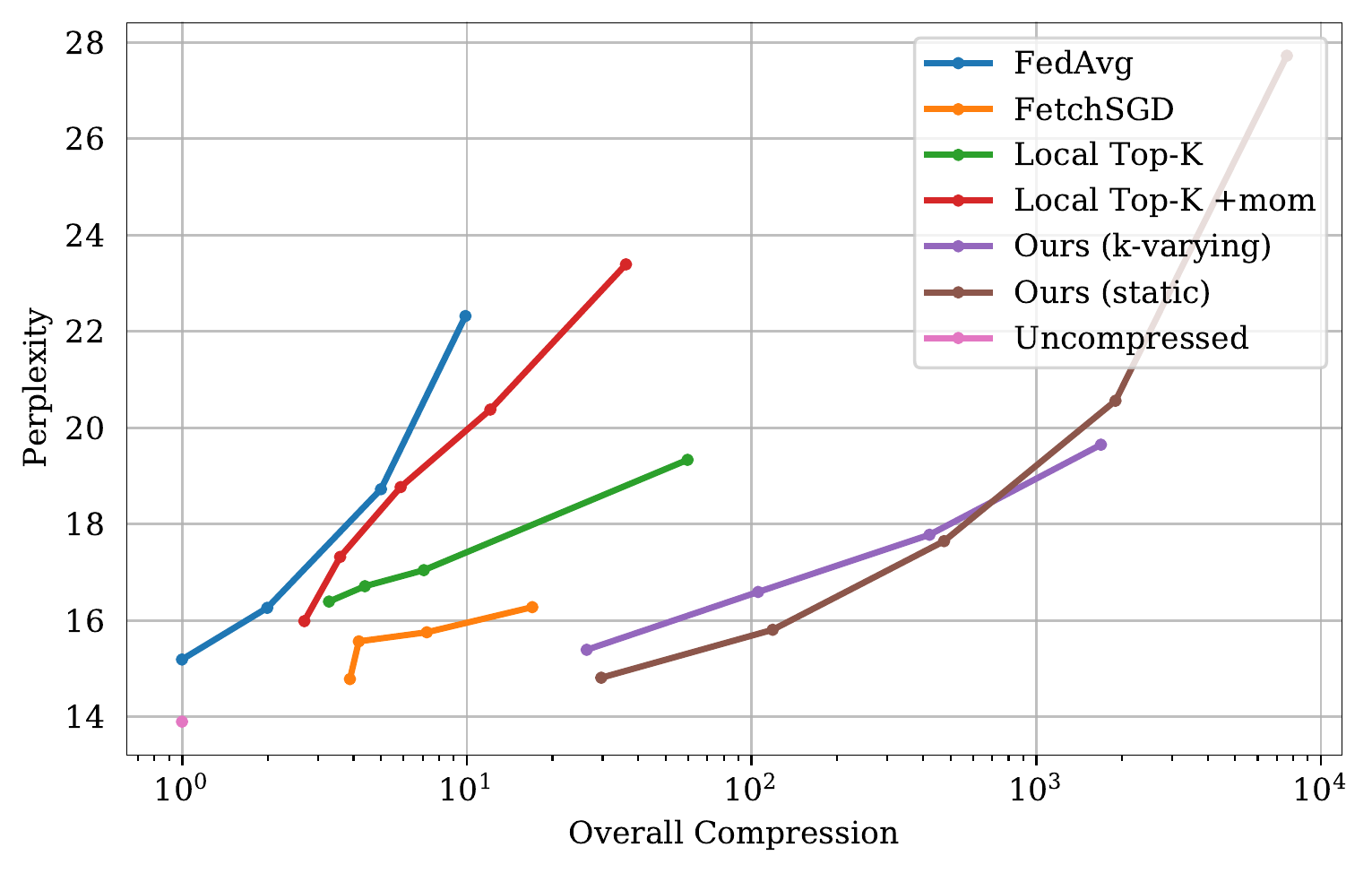} }}
    \qquad
    \subfloat[\centering Accuracy on SST-2  ]{{\includegraphics[width=0.4\textwidth]{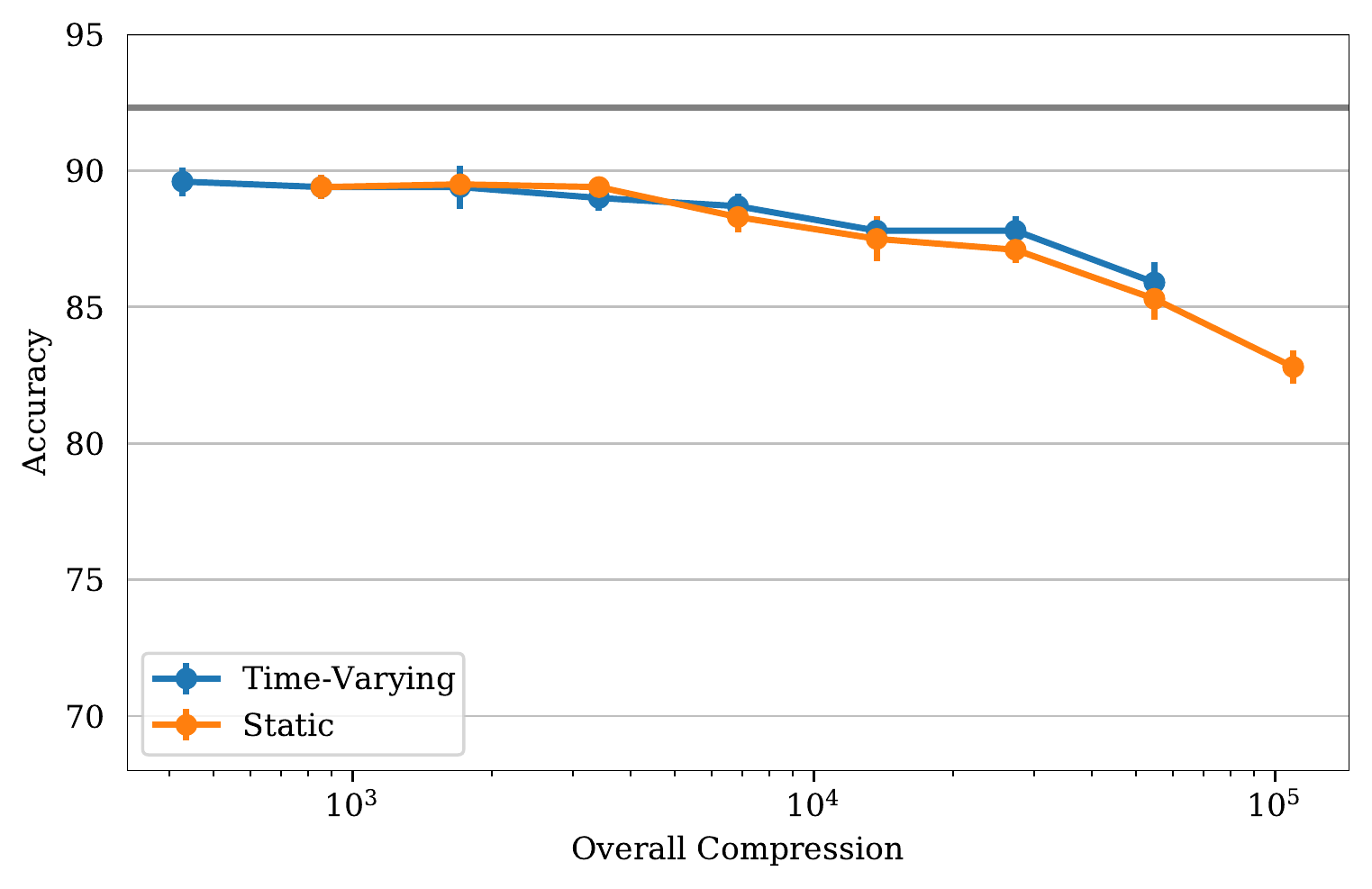} }}%
    \caption{Results on NLP benchmarks. $K$-subspace and static compression both strongly outperform all other methods, though $K$-subspace has the added benefit of much lower upload compression (not shown). 
    For the SST-2 results, error bars show the standard error of performance calculated over five runs with different random seeds.
    }
    \label{fig:nlpfig}
    \vspace{-4mm}
\end{figure}

\section{Experiments}\label{sec:exps}

We evaluate our method across three benchmarks: two from NLP (language modeling and text classification) and one from computer vision (image classification). 
As with previous works \cite{DBLP:conf/icml/RothchildPUISB020,mcmahan2017communication}, we simulate a federated setting in order to scale to large numbers of clients (upwards of $10,000$). We perform experiments in both non-IID and IID settings.

\paragraph{Image Classification (ResNet-9 on CIFAR-10)}

First, we consider image classification on CIFAR-10, a dataset of 50,000 $32\times32$px images. We use the same experimental setup as \cite{DBLP:conf/icml/RothchildPUISB020}: we split the data between 10,000 clients in a non-IID fashion, such that each client only has data from a single class. At each step, we sample 100 clients at random, such that each gradient step corresponds to 500 images. We perform 24 rounds  of communication between all clients (i.e. 24 epochs). 

We use a ResNet-9 architecture with 6,570,880 trainable parameters for our fair comparison to previous work. Note that the model does not have batch normalization, as it would not make sense in a setting where each client has so few examples. Due to the substantial number of epochs performed here, we experiment with both static and time-varying gradient compression ($K$-subspace compression is better suited to settings involving fewer rounds of communication). We experiment with intrinsic dimensions from 4000 to 256000.  

Our results are shown in \Cref{fig:cvfig}. Whereas FedAvg and Top-K struggle at even modest compression rates (e.g. $3\times$), the intrinsic gradient compression methods deliver strong performance at much larger compression rates. The intrinsic methods outperform the current state-of-the-art gradient compression method, FetchSGD~\cite{DBLP:conf/icml/RothchildPUISB020}, by a large margin, and easily scales to high compression rates (e.g. $100\times$). Finally, we see that time-varying intrinsic compression generally outperforms static compression for the same communication cost.

\paragraph{Text Classification (BERT on SST-2)}

Next, we consider text classification on the Stanford Sentiment Treebank-v2 (SST-2) dataset \cite{sst2}, a common sentiment analysis dataset. For this experiment, we consider an IID data split into 50 and 500 clients, respectively. We employ the popular BERT \cite{devlin-etal-2019-bert} architecture with 109M parameters and we use intrinsic dimensions from 200 to 25600. The purpose of this experiment is to push the limits of gradient compression; we project the 109M-dimension BERT gradients into as few as 200 dimensions. 

Our results are given in \Cref{fig:nlpfig}. First, in agreement with \cite{aghajanyan2020intrinsic}, we find that it is possible to achieve remarkably high compression ratios for text classification: we get nearly full performance even when compressing the 109M-dimension parameter vector into an intrinsic space of dimension 16,384. Furthermore, we find that time-varying intrinsic gradient compression consistently outperforms static intrinsic gradient compression at the same compression rate. 

\paragraph{Language Modeling (GPT-2 on PersonaChat)}

Lastly, we consider language modeling on the PersonaChat~\cite{zhang2018personalizing} dataset. The dataset has a non-IID split into 17,568 clients in which each client is assigned all data corresponding to given personality; as a result, it is widely used in federated learning simulations. We perform language modeling using the GPT-2 transformer architecture (124M parameters) and conduct two rounds of training across the clients (i.e. two epochs). Due to the low number of training rounds, it is natural to apply \textit{static} and $K$-subspace gradient compression (we use $K=8$).\footnote{Time-varying compression does not make sense here, as its benefit is derived from the setting where there are many rounds of communication between the clients.}

Our results are shown in \Cref{fig:nlpfig}. Overall, intrinsic dimension-based gradient compression vastly outperforms a wide range of prior approaches to reducing communication in federated learning. On the low-compression end of the spectrum, we obtain nearly full performance with superior compression rates to the state-of-the-art FetchSGD~\cite{DBLP:conf/icml/RothchildPUISB020}. On the high-compression end of the spectrum, we scale better than previous approaches. For example, we obtain a perplexity of around 20 even with an extremely high compression rate of 1898$\times$. 

Finally, we see that $K$-subspace intrinsic compression performs similarly to (or slightly worse) than static compression at the same level of overall compression. However, if it is more important to conserve upload bandwidth than download bandwidth, then $K$-subspace intrinsic gradient compression significantly outperforms static intrinsic gradient compression (see \Cref{table:personachat}). 

\paragraph{Gradient Reconstruction: Data Privacy Experiment}

One of the primary motivations of federated learning is the desire for individual clients to be able to retain data privacy while still participating in model training. However, prior work \cite{DBLP:conf/nips/ZhuLH19} has shown that if the client sends their full local model update to the server, it is sometimes possible to approximately reconstruct their local data from the model update. We investigate the extent to which an attacker can reconstruct a client's data given a \textit{compressed} gradient update, and we find that our compression helps to mitigate this reconstruction problem. Full details are included in \Cref{app:gradient_reconstruction} due to space constraints.

%% file: concl.tex
\vspace{-2mm}
\section{Conclusion}\label{sec:concl}

We propose a family of intrinsic gradient compression algorithms for federated learning. This family includes static compression, which performs remarkably well despite its simplicity, $K$-subspace compression, which is optimized for upload bandwidth, and time-varying compression, which improves performance by changing the intrinsic subspace over time. We provide theoretical results for our algorithms and demonstrate their effectiveness through numerous large-scale experiments. We hope that our results help make the real-world deployment of large-scale federated learning systems more feasible.

%% file: proofs.tex
\section{Proofs Omitted in the Main Text}\label{appa:proofs}

\subsection{Proof of \Cref{thm:static}}\label{appa:static}

 First, we show that $h(\theta') := g(A\theta' + \theta_0)$ is convex in $\theta'$. 

\begin{lemma}
$h$ is convex. 
\end{lemma}

\begin{proof}
We have 
\begin{align*}
    h(\lambda\theta_1' + (1 - \lambda)\theta_2') &= g(A(\lambda\theta_1' + (1 - \lambda)\theta_2') + \theta_0) \\ 
    &\le \lambda g(A\theta_1' + \theta_0) + (1 - \lambda) g(A\theta_2' + \theta_0) \\
    &= \lambda h(\theta_1') + (1 - \lambda) h(\theta_2')
\end{align*}
and we may conclude.
\end{proof}

We can now write \[ h(\bm{x}_t) - g^{\star} = (h(\bm{x}_t) - h^{\star}) + (h^{\star} - g^{\star}) \]

We can bound the first term with a result from \cite{scaffold} because $h$ is convex, and thus classical convex optimization algorithms will converge quickly (namely, within $\tilde{O}(\sigma^2 / \epsilon^2)$ steps). The second term is bounded by our assumption on the intrinsic dimension of the function $g$. With at least probability $1 - \delta$, we have that $h^{\star} - g^{\star}$ is at most $\rho (g(\theta_0) - g^{\star})$.

\subsection{Proof of \Cref{thm:kvary}}

In this part of the problem, it is not immediately clear how to fit it into the existing SGD framework. First, to parametrize $h$ we use \[ A = \begin{bmatrix} A_1 & A_2 & \ldots & A_k \end{bmatrix}. \] and take $h(\theta') = g(A\theta' + \theta_0)$. The correct gradient of this function is $A^{\top} \bm{g}_t$, where $\bm{g}_t$ is the true gradient. However, now define \[ A_i' = \begin{bmatrix} 0 & \ldots & \underbrace{A^{(i)}}_{i\text{th index}} & \ldots 0  \end{bmatrix}  \]

Then, we claim that our algorithm is equivalent to using $kA_i'^{\top}\bm{g}_t$ as an unbiased gradient estimate. Thus, the SGD equation looks like $\theta'_{t+1} = \theta'_{t} - A_i'^{\top} \bm{g}_t$, and after multiplying both sides by the matrix $A$ we get \[ \theta_{t+1} = \theta_t - AA_i'^{\top} \bm{g}_t = \theta_t - A_i'A_i'^{\top}\bm{g}_t = \theta_t - A^{(i)}A^{(i)\top}\bm{g}_t, \] which matches our algorithm for $K$-subspace compression.

It remains to compute the variance of the gradients $A_i'^{\top}\bm{g}_t$, which is used in the SGD bound. We obtain that $\BE[\bm{g}_t^{\top}A_i'A_i^{'\top}\bm{g}_t] = k\BE[\bm{g}_t^{\top}AA^{\top}\bm{g}_t]$. Note that 

\begin{align*}
    \mathrm{Var}[A_i^{\top}\bm{g}_t] &= \mathbb{E}[\bm{g}_t^{\top}A_iA_i^{\top}\bm{g}_t] - (\mathbb{E}[A_i^{\top}\bm{g}_t])^2 \\
    &= k((\mathbb{E}[A_i^{\top}\bm{g}_t])^2 + \mathrm{Var}[A_i^{\top} \bm{g}_t]) - (\mathbb{E}[A_i^{\top}\bm{g}_t])^2 \\
    &\le k((\mathbb{E}[A_i^{\top}\bm{g}_t])^2 + \mathrm{Var}[A_i^{\top} \bm{g}_t]) \\
    &\le k\pa{1 + \frac{1}{C}}\mathrm{Var}[A^{\top} \bm{g}_t])
\end{align*}

Thus, we have that the true variance, given the ratio, is at most $K(1 + C) / C = K(1 + 1/C)$ times the original variance. The rest of the analysis is exactly the same as \Cref{appa:static}, and we may conclude. 

\subsection{Proof of \Cref{thm:timevary}}

Here, we repeatedly apply \Cref{thm:static} by using the fact that we essentially sample fresh directions each time. Intuitively, the time-varying design implies that each new subspace choice is a fresh opportunity to get closer to the optimum. Each epoch lets us get closer and closer to the desired optimum. 

We have that after $\sigma^2 / E\epsilon^2$ iterations from \cite{scaffold}, the loss is at most $r(g(\theta_0) - g^{\star})$, where $r(x) := \rho x + \epsilon \sqrt{E}$. By repeatedly applying this result, with probability at least $(1 - \delta)^{E}$, the final loss is at most $r^{E}(g(\theta_0) - g^{\star})$, where \[ r^{E}(x) = \rho^{E} x + (\rho^{E-1}\epsilon\sqrt{E} + \ldots + \epsilon \sqrt{E}) \le \rho^{E} x + \frac{\epsilon\sqrt{E}}{1 - \rho}, \] and we may conclude.

%% file: algo3.tex
\section{$K$-subspace Intrinsic Gradient Compression}

This is given in \Cref{alg:FedkTVSC}.

%% file: additional_exps.tex
\section{Additional Related Work}\label{app:additional_related_work}

\subsection{Intrinsic Dimensionality} As mentioned in the main paper, the concept of measuring the intrinsic dimensional of loss landscapes was introduced by \cite{li2018measuring}. \cite{li2018measuring} consider optimizing a $D$-parameter model in a random $d$-dimensional subspace of the full parameter space. They define the intrinsic dimension of the optimization problem as the minimum dimension $d$ for which a solution to the problem can be found, where a ``solution'' refers attaining a certain percentage of the maximum possible validation accuracy (i.e. the validation accuracy obtained by optimizing in all $D$ dimensions). They use a fixed cut-off of $90$\% accuracy for their experiments. \cite{aghajanyan2020intrinsic} apply these ideas in the setting of finetuning NLP models. 

A number of works have tried to measure the intrinsic dimension of datasets, rather than objective landscapes. \cite{NIPS2004_74934548} introduced a maximum likelihood approach to estimating intrinsic dimensionality based on nearest-neighbors, while \cite{CERUTI20142569} employed angle and norm-based similarity. 

Finally, some works have tried to measure the intrinsic dimensionality of image representations and datasets. \cite{gong2019intrinsic} finds that the representations produced by popular image and face representation learning models (ResNet-50 and SphereFace) have quite low intrinsic dimensionalities (16 and 19, respectively). Along similar lines, \cite{pope2021the} showed that popular image datasets (MNIST, CIFAR 10, ImageNet) also have low intrinsic dimensionality. 

\subsection{Model Pruning}
There has been great interest in compressing models by using fewer weights, starting with the work of \cite{hinton2015distilling, han2015deep}. One related work is \emph{Diff Pruning} \cite{guo2020parameter}, which constrains the number of weights that can be changed from a pretrained model. In essence, diff pruning attempts to solve an $L^{0}$ minimization problem on the weights of the model, and approaches this by means of a relaxation to a problem that is more amenable to a standard analysis. 

A number of other works have explored the idea of finetuning by only modifying a subset of a model's parameters. 
\cite{ravfogel2021bitfit} finetunes only the layer biases, whereas \cite{houlsby2019parameter} introduces the concept of low-parameter adapters between each layer. Compared to \cite{ravfogel2021bitfit} our method is far more flexible, allowing any number of parameters to be changed. Compared to \cite{houlsby2019parameter} our methods are architecture-independent, and can be applied to any model.

\paragraph{Federated Learning}

Federated learning is generally concerned with the distributed training of machine learning models across many devices, each of which holds private data. Many aspects of this federated setup are separate subfields of research, including how to ensure the privacy of client-held data \cite{Xie2020DBA,bhagoji2019analyzing}, how to deal with heterogeneous data and networks \cite{li2020federated,li2020convergence,yu2020federated}, how to reconcile weights/gradients from multiple clients \cite{li2020federated,wang2020federated,pmlr-v119-li20g}, how to manage clients in a fault-tolerant manner, deployment on mobile/iot devices \cite{chaoyanghe2020fedml}, and fairness \cite{mohri2019agnostic}. 

The classic FedAvg~\cite{mcmahan2017communication} algorithm communicates model updates after multiple local training iterations.  FedProx~\cite{li2020federated} generalized and re-parametrized FedAvg, and FedMA~\cite{wang2020federated} improved this approach by matching and averaging hidden layers of networks with similar activations at each communication round. 
Additionally, FedAwS~\cite{yu2020federated} considered federated averaging in the case where each client has data from only a single class.

\section{Further Experimental Details and Analysis}\label{app:additional}
In the main paper, we included a number of figures demonstrating our performance in comparison to prior work. Here, we include tables with our precise results for clarity and in order to facilitate future comparison with our work. 

\subsection{General Implementation Details}

We perform our language modeling experiments on 8 RTX 6000 GPUs and our image/text classification experiments on 1 RTX 6000 GPU. Regarding the intrinsic gradient compression matrices $A_i$, we employ the FastFood method described in \Cref{sec:fedgradient_choice} using a CUDA implementation of the fast Walsh-Hadamard transform from \cite{thomas2018learning}.

\subsection{Further PersonaChat Analysis}

First, we give more details on the PersonaChat dataset, which were omitted from the main paper due to space constraints. The PersonaChat dataset \cite{zhang2018personalizing} was collected by first giving imaginary personas (defined by a set of 5 sentences) to Amazon Mechanical Turk workers and asking them to take on those personas. Then, the system paired workers and asked them to discuss. Since the personas were imaginary and no personally identifiable information was exchanged (in particular, the workers were explicitly told to not use personally identifiable information) the dataset does not contain personally identifiable information. The dataset has a non-IID split into 17,568 clients in which each client is assigned all data corresponding to given personality; as a result, it is widely used in federated learning simulations. We perform language modeling using the GPT-2 transformer architecture (124M parameters). We perform \textit{static} and $K$-subspace gradient compression using intrinsic dimensions of 16384, 65536, 262144, 1048576, and 4194304. 

We show full results on PersonaChat below, complete with upload and download compression. Overall compression is calculated as average compression over both upload and download. We compare with FedAvg~\cite{mcmahan2017communication}, Top-K, and FetchSGD~\cite{DBLP:conf/icml/RothchildPUISB020}. FedAvg is the baseline federated learning approach involving sending and averaging weights. Top-K refers to sending the top gradients, sorted by magnitude. FetchSGD compresses the weights with sketching.

Our method significantly outperforms competing approaches across the board. We obtain an accuracy close to that of uncompressed optimization using 29.7$\times$ overall compression; FedAvg and Top-K both fail to achieve such strong results, while FetchSGD does so at a significantly lower compression rate. 

Next we compare static and K-varying intrinsic gradient compression. When comparing overall compression rates, static compression is slightly better than K-varying compression. However, K-varying compression is optimized for low upload bandwidth; it obtains much better upload compression rates than static compression at the same accuracy. For example, K-varying compression with $k=8$ and $d=65536$ yields perplexity $17.6$ at upload compression $1900\times$, whereas static compression with $d=262144$ yields perplexity $17.4$ at upload compression $475\times$.

\input{tables/table_personachat}

\subsection{Further SST-2 Details and Analysis}

\input{tables/table_glue}

Regarding the experimental setup, we perform 30 rounds (i.e. 30 epochs) of training for all compressed runs, while we perform 6 for the uncompressed baseline (as it converges more quickly). Federated learning experiments has previously been criticized for being challenging to reproduce; as a result, we perform each run five times over different random seeds. Due to the substantial number of epochs performed here, it is natural to apply static and time-varying intrinsic gradient compression. We use intrinsic dimensions of 200, 400, 800, $\dots$, 25600. 

In \Cref{table:glue}, we show full results for the SST-2 dataset with static and time-varying gradient compression for a range of intrinsic dimensions. We include in this experiment an demonstration of the robustness of our method to variation in random seeds; we run each experiment five times using separate random seeds (i.e. different intrinsic subspaces and model initializations). We report standard errors in \Cref{table:glue} and include \Cref{fig:nlpfig} with error bars in the main paper. Overall variability is quite low. 

We also see that time-varying intrinsic gradient compression outperforms static intrinsic compression, especially for low intrinsic dimensions. For example, time-varying compression at $d=200$ outperforms static compression with $d=400$, and time-varying compression with $d=400$ outperforms static compression with $d=800$.

%% file: tables/table_personachat.tex
\begin{table*}
\begin{center}

\addtolength{\leftskip} {-1.2cm}
\addtolength{\rightskip}{-1.2cm}
\begin{tabular}{clrcccc}
\toprule
{}                                      &  Name                      &  Intrinsic Dim. &  PPL  &  Up. Comp. &  Down. Comp. &  Total Comp. \\ \midrule
                                        &  Uncompressed              &                 &  13.9 &  1         &  1           &  1    \\ \midrule
\cite{mcmahan2017communication}         &  FedAvg (2 local iters)    &                 &  16.3 &  2         &  2           &  2    \\
\cite{mcmahan2017communication}         &  FedAvg (5 local iters)    &                 &  20.1 &  5         &  5           &  5    \\ \midrule
                                        &  Local Top-K ($k=50,000$)  &                 &  19.3 &  30.3      &  2490        &  60   \\
                                        &  Local Top-K ($k=500,000$) &                 &  17.1 &  3.6       &  248         &  7.1  \\ \midrule
\cite{DBLP:conf/icml/RothchildPUISB020} &  FetchSGD ($k=25,000$)     &                 &  \textbf{14.8} &  3.8       &  100         &  7.3  \\
\cite{DBLP:conf/icml/RothchildPUISB020} &  FetchSGD ($k=50,000$)     &                 &  15.8 &  2.4       &  10          &  3.9  \\ \midrule
                                        &  Ours (static)             &  16384        &  27.7 &  7595      &  7595        &  7595 \\
                                        &  Ours ($K$-subspace)          &  16384        &  19.6 &  7595      &   949     &  1688 \\
                                        &  Ours (static)             &  65536        &  20.6 &  1900      &  1900        &  1900 \\
                                        &  Ours ($K$-subspace)          &  65536        &  17.8 &  1900     &  237        &  422 \\
                                        &  Ours (static)             &  262144       &  17.6 &  475      &  475        &  475 \\
                                        &  Ours ($K$-subspace)          &  262144       &  16.6 &  475      &    59.3      &  105 \\
                                        &  Ours (static)             &  1048576      &  15.8 &  119      &  119        &  119 \\
                                        &  Ours ($K$-subspace)          &  1048576      &  15.4 &  119      &  14.8     &  26.3 \\
                                        &  Ours (static)             &  4194304      &  \textbf{14.8} &  29.7      &  29.7        &  29.7 \\
\bottomrule
\end{tabular}
\vspace{4mm}
\caption{Results of our method and comparison to prior work, including the state-of-the-art in gradient compression (FetchSGD). The table shows language modeling perplexity (lower is better) and compression rates (higher is better). We show upload, download, and total compression rates. For our intrinsic gradient compression results, we show static and $K$-subspace compression for a range of dimensions between $16386$ and $4194304$. For $K$-subspace compression we use $K=8$. }
\end{center}
\vspace{-4mm}
\label{table:personachat}
\end{table*}

%% file: tables/table_glue.tex
\begin{table*}
\centering
\begin{tabular}{lcccc} \toprule 
Intrinsic Dim. &  200                 &  400               &  800               &  1,600                 \\ \midrule
Static         &  82.8 ($\pm 0.69$)   &  85.3 ($\pm 0.89$) &  87.1 ($\pm 0.57$) &  87.5 ($\pm 0.94$)    \\ \midrule
Time-Varying   & 85.9 ($\pm 0.85$) &  87.8 ($\pm 0.61$) &  87.8 ($\pm 0.59$) &  88.7 ($\pm 0.54$)    \\ \bottomrule
\end{tabular}

\vspace{5mm}

\begin{tabular}{lcccc}\toprule
Intrinsic Dim. &  3,200             &  6,400             &  12,800            &  25,600                 \\ \midrule
Static         &  88.3 ($\pm 0.65$) &  89.4 ($\pm 0.33$) &  89.5 ($\pm 0.21$) &  89.5 ($\pm 0.21$)      \\ \midrule
Time-Varying   &  89.0 ($\pm 0.53$) &  89.4 ($\pm 0.91$) &  89.4 ($\pm 0.19$) &  89.4 ($\pm 0.19$)      \\ \bottomrule
\end{tabular}
\vspace{3mm}
\caption{Accuracy and standard error of a BERT model trained on the Stanford Sentiment Treebank v2 (SST-2) for varying intrinsic dimensions. We calculate the standard error over five trials with different random seeds. We see that for fixed dimension, time-varying intrinsic gradient compression outperforms static intrinsic gradient compression.}
\label{table:glue}
\end{table*}

%% file: gradient_reconstruction.tex
\section{Gradient Reconstruction: Data Privacy Experiment} \label{app:gradient_reconstruction}

\begin{figure}%
    \centering
    \subfloat[\centering Input]{{\includegraphics[width=0.3\textwidth]{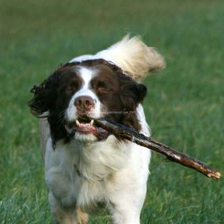}}}%
    \quad
    \subfloat[\centering Reconstruction from full gradient. ]{{\includegraphics[width=0.3\textwidth]{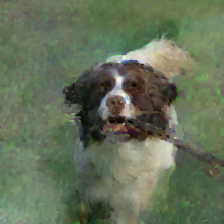}}}%
    \quad
    \subfloat[\centering Reconstruction from gradient with intrinsic compression. ]{{\includegraphics[width=0.3\textwidth]{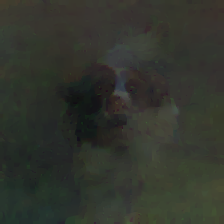}}}%
    \caption{Image reconstruction from gradients with and without our intrinsic gradient compression method. On the left, we show the original image. In the center, we show the result of reconstructing the image from a single gradient from a ResNet-152 model (60M parameters), produced using the method of \cite{DBLP:conf/nips/ZhuLH19}. On the right, we show the result of the same image reconstruction method applied to an gradient compressed by our algorithm using intrinsic dimension 65,536.}
    \label{fig:inverse_gradient}
\end{figure}

Data privacy is one of the central motivations of federated learning. 

However, a number of works have shown that if the client does not have a large amount of data and the client sends back their full local gradient, it is possible to approximately reconstruct their local data from the model. This is a significant problem, because their data would then effectively be visible to the central server and any attackers that intercept their communications. 

This is a significant problem, because their data would then effectively be visible to the central server and any attackers that intercept their communications.

Here, we show that compressing gradients with our approach can mitigate this problem. 
Specifically, we check if our compressed gradients can be reconstructed with the iterative procedure proposed by \cite{DBLP:conf/nips/ZhuLH19}, which takes a gradient and a model and tries to recover an image.
As in \cite{DBLP:conf/nips/ZhuLH19}, we use a ResNet-152 model on a randomly selected image from ImageNet and run for 24,000 iterations (by which time the method has converged). We reconstruct the image both from the full gradient (the center image) and from a the intrinsically-compressed image (the right image) with intrinsic dimension 65,536. 

As seen in \Cref{fig:inverse_gradient}, given the full gradient it is possible to obtain a fairly good reconstruction of the image. By contrast, with our method, the reconstruction is visually much less similar to the original image. 
Of course, our method does not solve the problem entirely; an outline of the dog in the image is still visible because the compressed gradient still contains some information about the local data. To solve the issue entirely, it would be necessary to use a method such as differential privacy. 

%% file: main.bbl
\begin{thebibliography}{10}

\bibitem{ajiheafield2017sparse}
Alham~Fikri Aji and Kenneth Heafield.
\newblock Sparse communication for distributed gradient descent.
\newblock In {\em Proceedings of the 2017 Conference on Empirical Methods in
  Natural Language Processing}, pages 440--445, Copenhagen, Denmark, September
  2017. Association for Computational Linguistics.

\bibitem{Sattler2020RobustAC}
Felix Sattler, Simon Wiedemann, K.~M{\"u}ller, and W.~Samek.
\newblock Robust and communication-efficient federated learning from non-i.i.d.
  data.
\newblock {\em IEEE Transactions on Neural Networks and Learning Systems},
  31:3400--3413, 2020.

\bibitem{lin2018deep}
Yujun Lin, Song Han, Huizi Mao, Yu~Wang, and Bill Dally.
\newblock Deep gradient compression: Reducing the communication bandwidth for
  distributed training.
\newblock In {\em International Conference on Learning Representations}, 2018.

\bibitem{DBLP:conf/icml/RothchildPUISB020}
Daniel Rothchild, Ashwinee Panda, Enayat Ullah, Nikita Ivkin, Ion Stoica,
  Vladimir Braverman, Joseph Gonzalez, and Raman Arora.
\newblock Fetchsgd: Communication-efficient federated learning with sketching.
\newblock In {\em Proceedings of the 37th International Conference on Machine
  Learning, {ICML} 2020, 13-18 July 2020, Virtual Event}, volume 119 of {\em
  Proceedings of Machine Learning Research}, pages 8253--8265. {PMLR}, 2020.

\bibitem{li2018measuring}
Chunyuan Li, Heerad Farkhoor, Rosanne Liu, and Jason Yosinski.
\newblock Measuring the intrinsic dimension of objective landscapes.
\newblock In {\em International Conference on Learning Representations}, 2018.

\bibitem{aghajanyan2020intrinsic}
Armen Aghajanyan, Sonal Gupta, and Luke Zettlemoyer.
\newblock Intrinsic dimensionality explains the effectiveness of language model
  fine-tuning.
\newblock In Chengqing Zong, Fei Xia, Wenjie Li, and Roberto Navigli, editors,
  {\em Proceedings of the 59th Annual Meeting of the Association for
  Computational Linguistics and the 11th International Joint Conference on
  Natural Language Processing, {ACL/IJCNLP} 2021, (Volume 1: Long Papers),
  Virtual Event, August 1-6, 2021}, pages 7319--7328. Association for
  Computational Linguistics, 2021.

\bibitem{stich2020error}
Sebastian~U Stich and Sai~Praneeth Karimireddy.
\newblock The error-feedback framework: Better rates for sgd with delayed
  gradients and compressed updates.
\newblock {\em Journal of Machine Learning Research}, 21:1--36, 2020.

\bibitem{ef21}
Peter Richtárik, Igor Sokolov, and Ilyas Fatkhullin.
\newblock {EF21}: {A} {New}, {Simpler}, {Theoretically} {Better}, and
  {Practically} {Faster} {Error} {Feedback}.
\newblock {\em arXiv:2106.05203 [cs, math, stat]}, June 2021.
\newblock arXiv: 2106.05203.

\bibitem{albasyoni2020optimal}
Alyazeed Albasyoni, Mher Safaryan, Laurent Condat, and Peter Richt{\'a}rik.
\newblock Optimal gradient compression for distributed and federated learning.
\newblock {\em arXiv preprint arXiv:2010.03246}, 2020.

\bibitem{mcmahan2017communication}
Brendan McMahan, Eider Moore, Daniel Ramage, Seth Hampson, and Blaise~Aguera
  y~Arcas.
\newblock Communication-efficient learning of deep networks from decentralized
  data.
\newblock In {\em Artificial Intelligence and Statistics}, pages 1273--1282.
  PMLR, 2017.

\bibitem{shi2019understanding}
Shaohuai Shi, Xiaowen Chu, Ka~Chun Cheung, and Simon See.
\newblock Understanding top-k sparsification in distributed deep learning,
  2019.

\bibitem{konecny2018federated}
Jakub Konecny, H.~Brendan McMahan, Felix~X. Yu, Ananda~Theertha Suresh, Dave
  Bacon, and Peter Richtárik.
\newblock Federated learning: Strategies for improving communication
  efficiency, 2018.

\bibitem{vkj2019powerSGD}
Thijs Vogels, Sai~Praneeth Karimireddy, and Martin Jaggi.
\newblock {{PowerSGD}: Practical Low-Rank Gradient Compression for Distributed
  Optimization}.
\newblock In {\em NeurIPS 2019 - Advances in Neural Information Processing
  Systems}, 2019.

\bibitem{DBLP:conf/icml/LeSS13}
Quoc~V. Le, Tam{\'{a}}s Sarl{\'{o}}s, and Alexander~J. Smola.
\newblock Fastfood - computing hilbert space expansions in loglinear time.
\newblock In {\em Proceedings of the 30th International Conference on Machine
  Learning, {ICML} 2013, Atlanta, GA, USA, 16-21 June 2013}, volume~28 of {\em
  {JMLR} Workshop and Conference Proceedings}, pages 244--252. JMLR.org, 2013.

\bibitem{stich2013optimization}
Sebastian~U Stich, Christian~L Muller, and Bernd Gartner.
\newblock Optimization of convex functions with random pursuit.
\newblock {\em SIAM Journal on Optimization}, 23(2):1284--1309, 2013.

\bibitem{ssobound}
Long Chen, Xiaozhe Hu, and Steven Wise.
\newblock Convergence analysis of the fast subspace descent method for convex
  optimization problems.
\newblock {\em Mathematics of Computation}, 89(325):2249--2282, 2020.

\bibitem{sst2}
Richard Socher, Alex Perelygin, Jean Wu, Jason Chuang, Christopher~D Manning,
  Andrew~Y Ng, and Christopher Potts.
\newblock Recursive deep models for semantic compositionality over a sentiment
  treebank.
\newblock In {\em Proceedings of the 2013 conference on empirical methods in
  natural language processing}, pages 1631--1642, 2013.

\bibitem{devlin-etal-2019-bert}
Jacob Devlin, Ming-Wei Chang, Kenton Lee, and Kristina Toutanova.
\newblock {BERT}: Pre-training of deep bidirectional transformers for language
  understanding.
\newblock In {\em Proceedings of the 2019 Conference of the North {A}merican
  Chapter of the Association for Computational Linguistics: Human Language
  Technologies, Volume 1 (Long and Short Papers)}, pages 4171--4186,
  Minneapolis, Minnesota, June 2019. Association for Computational Linguistics.

\bibitem{zhang2018personalizing}
Saizheng Zhang, Emily Dinan, Jack Urbanek, Arthur Szlam, Douwe Kiela, and Jason
  Weston.
\newblock Personalizing dialogue agents: I have a dog, do you have pets too?
\newblock In {\em Proceedings of the 56th Annual Meeting of the Association for
  Computational guistics (Volume 1: Long Papers)}, pages 2204--2213, 2018.

\bibitem{DBLP:conf/nips/ZhuLH19}
Ligeng Zhu, Zhijian Liu, and Song Han.
\newblock Deep leakage from gradients.
\newblock In Hanna~M. Wallach, Hugo Larochelle, Alina Beygelzimer, Florence
  d'Alch{\'{e}}{-}Buc, Emily~B. Fox, and Roman Garnett, editors, {\em Advances
  in Neural Information Processing Systems 32: Annual Conference on Neural
  Information Processing Systems 2019, NeurIPS 2019, December 8-14, 2019,
  Vancouver, BC, Canada}, pages 14747--14756, 2019.

\bibitem{scaffold}
Sai~Praneeth Karimireddy, Satyen Kale, Mehryar Mohri, Sashank Reddi, Sebastian
  Stich, and Ananda~Theertha Suresh.
\newblock Scaffold: Stochastic controlled averaging for federated learning.
\newblock In {\em International Conference on Machine Learning}, pages
  5132--5143. PMLR, 2020.

\bibitem{NIPS2004_74934548}
Elizaveta Levina and Peter Bickel.
\newblock Maximum likelihood estimation of intrinsic dimension.
\newblock In L.~Saul, Y.~Weiss, and L.~Bottou, editors, {\em Advances in Neural
  Information Processing Systems}, volume~17. MIT Press, 2005.

\bibitem{CERUTI20142569}
Claudio Ceruti, Simone Bassis, Alessandro Rozza, Gabriele Lombardi, Elena
  Casiraghi, and Paola Campadelli.
\newblock Danco: An intrinsic dimensionality estimator exploiting angle and
  norm concentration.
\newblock {\em Pattern Recognition}, 47(8):2569--2581, 2014.

\bibitem{gong2019intrinsic}
Sixue Gong, Vishnu~Naresh Boddeti, and Anil~K Jain.
\newblock On the intrinsic dimensionality of image representations.
\newblock In {\em Proceedings of the IEEE/CVF Conference on Computer Vision and
  Pattern Recognition}, pages 3987--3996, 2019.

\bibitem{pope2021the}
Phil Pope, Chen Zhu, Ahmed Abdelkader, Micah Goldblum, and Tom Goldstein.
\newblock The intrinsic dimension of images and its impact on learning.
\newblock In {\em International Conference on Learning Representations}, 2021.

\bibitem{hinton2015distilling}
Geoffrey Hinton, Oriol Vinyals, and Jeff Dean.
\newblock Distilling the knowledge in a neural network.
\newblock {\em arXiv preprint arXiv:1503.02531}, 2015.

\bibitem{han2015deep}
Song Han, Huizi Mao, and William~J Dally.
\newblock Deep compression: Compressing deep neural networks with pruning,
  trained quantization and huffman coding.
\newblock {\em arXiv preprint arXiv:1510.00149}, 2015.

\bibitem{guo2020parameter}
Demi Guo, Alexander~M Rush, and Yoon Kim.
\newblock Parameter-efficient transfer learning with diff pruning.
\newblock In {\em ACL}, 2021.

\bibitem{ravfogel2021bitfit}
Elad Ravfogel, Shauli Ben-Zaken, and Yoav Goldberg.
\newblock Bitfit: Simple parameter-efficient fine-tuning for transformer-based
  masked language-models.
\newblock {\em arXiv preprint}, 2021.

\bibitem{houlsby2019parameter}
Neil Houlsby, Andrei Giurgiu, Stanislaw Jastrzebski, Bruna Morrone, Quentin
  De~Laroussilhe, Andrea Gesmundo, Mona Attariyan, and Sylvain Gelly.
\newblock Parameter-efficient transfer learning for {NLP}.
\newblock In {\em Proceedings of the 36th International Conference on Machine
  Learning}, 2019.

\bibitem{Xie2020DBA}
Chulin Xie, Keli Huang, Pin-Yu Chen, and Bo~Li.
\newblock Dba: Distributed backdoor attacks against federated learning.
\newblock In {\em International Conference on Learning Representations}, 2020.

\bibitem{bhagoji2019analyzing}
Arjun~Nitin Bhagoji, Supriyo Chakraborty, Prateek Mittal, and Seraphin Calo.
\newblock Analyzing federated learning through an adversarial lens.
\newblock In {\em International Conference on Machine Learning}, pages
  634--643. PMLR, 2019.

\bibitem{li2020federated}
Tian Li, Anit~Kumar Sahu, Manzil Zaheer, Maziar Sanjabi, Ameet Talwalkar, and
  Virginia Smith.
\newblock Federated optimization in heterogeneous networks.
\newblock In {\em ML Sys}, 2020.

\bibitem{li2020convergence}
Xiang Li, Kaixuan Huang, Wenhao Yang, Shusen Wang, and Zhihua Zhang.
\newblock On the convergence of fedavg on non-iid data.
\newblock In {\em International Conference on Learning Representations}, 2020.

\bibitem{yu2020federated}
Felix~X. Yu, Ankit~Singh Rawat, Aditya~Krishna Menon, and Sanjiv Kumar.
\newblock Federated learning with only positive labels, 2020.

\bibitem{wang2020federated}
Hongyi Wang, Mikhail Yurochkin, Yuekai Sun, Dimitris Papailiopoulos, and
  Yasaman Khazaeni.
\newblock Federated learning with matched averaging.
\newblock In {\em International Conference on Learning Representations}, 2020.

\bibitem{pmlr-v119-li20g}
Zhize Li, Dmitry Kovalev, Xun Qian, and Peter Richtarik.
\newblock Acceleration for compressed gradient descent in distributed and
  federated optimization.
\newblock In Hal~Daumé III and Aarti Singh, editors, {\em Proceedings of the
  37th International Conference on Machine Learning}, volume 119 of {\em
  Proceedings of Machine Learning Research}, pages 5895--5904. PMLR, 13--18 Jul
  2020.

\bibitem{chaoyanghe2020fedml}
Chaoyang He, Songze Li, Jinhyun So, Mi~Zhang, Hongyi Wang, Xiaoyang Wang,
  Praneeth Vepakomma, Abhishek Singh, Hang Qiu, Li~Shen, Peilin Zhao, Yan Kang,
  Yang Liu, Ramesh Raskar, Qiang Yang, Murali Annavaram, and Salman Avestimehr.
\newblock Fedml: A research library and benchmark for federated machine
  learning.
\newblock {\em arXiv preprint arXiv:2007.13518}, 2020.

\bibitem{mohri2019agnostic}
Mehryar Mohri, Gary Sivek, and Ananda~Theertha Suresh.
\newblock Agnostic federated learning.
\newblock In {\em International Conference on Machine Learning}, pages
  4615--4625. PMLR, 2019.

\bibitem{thomas2018learning}
Anna~T Thomas, Albert Gu, Tri Dao, Atri Rudra, and Christopher Re.
\newblock Learning compressed transforms with low displacement rank.
\newblock {\em Advances in neural information processing systems}, 2018:9052,
  2018.

\end{thebibliography}
